%% file: main.tex
\RequirePackage[svgnames]{xcolor}

\documentclass[11pt,logo,letterpaper]{berkeley}

\usepackage[utf8]{inputenc} %
\usepackage[T1]{fontenc}    %
\usepackage{lmodern}
\usepackage{textcomp}
\usepackage[svgnames]{xcolor}
\usepackage{asymptote}
\usepackage[comma,authoryear,compress]{natbib}
\usepackage{hyperref}
\usepackage{algorithm}
\usepackage{algorithmicx}
\usepackage{algpseudocode}
\usepackage{microtype}
\usepackage{graphicx}
\expandafter\def\csname ver@subfig.sty\endcsname{}
\usepackage{booktabs} %
\usepackage{float}
\usepackage{bigstrut}
\usepackage{tikz}
\usetikzlibrary{tikzmark}
\makeatletter
\newcommand*\myfontsize{%
  \@setfontsize\myfontsize{7}{8}%
}
\makeatother
\definecolor{myred}{rgb}{0.7, 0.3, 0.0}
\definecolor{myblue}{HTML}{054488}
\definecolor{mygreen}{HTML}{056b34}

\newcolumntype{R}[1]{>{\raggedleft\let\newline\\\arraybackslash\hspace{0pt}}m{#1}}

\usetikzlibrary{intersections}

\definecolor{darkgreen}{rgb}{0.0, 0.42, 0.24}

\usepackage{amsthm}
\usepackage{nicefrac}
\usepackage{subcaption}
\usepackage{caption}
\usepackage{graphicx}
\usepackage{cleveref}
\usepackage{bxcoloremoji}
\usepackage{listings}
\usepackage{float}
\usepackage{longtable}
\lstdefinestyle{python}{
    language=Python,
    basicstyle=\ttfamily\footnotesize,
    keywordstyle=\color{blue}\bfseries,
    commentstyle=\color{green},
    stringstyle=\color{red},
    numberstyle=\tiny\color{gray},
    showstringspaces=false,
    frame=single,
    breaklines=true,
    backgroundcolor=\color{lightgray!20}
}

\usepackage{hyperref}       %
\usepackage{url}            %
\usepackage{booktabs}       %
\usepackage{nicefrac}       %
\usepackage{microtype}      %
\usepackage{graphicx}
\usepackage{subcaption} 
\usepackage{wrapfig}
\usepackage{lipsum}
\usepackage{enumitem}
\usepackage{stackengine}
\usepackage[font=small,labelfont=bf]{caption}
\usepackage{color}
\usepackage{adjustbox}
\usepackage{tikz}
\usepackage{tcolorbox}
\usepackage{rotating}
\usepackage{makecell}
\usepackage{colortbl}    
\usepackage{multirow}    
\usepackage{pifont}      
\usepackage{xcolor}      
\usepackage{array}       
\usepackage{amsmath}
\usepackage{amsfonts}       %
\usepackage{xspace}
\usepackage{mathtools}
\definecolor{oursblue}{RGB}{230,240,255} 

\input{macro}

\newtcolorbox{AIbox}[2][]{aibox,title=#2,#1}
\definecolor{lightblue}{rgb}{0.22,0.45,0.70}%
\definecolor{Gray}{gray}{0.95}
\definecolor{Cornsilk}{rgb}{1.0, 0.97, 0.86}
\usepackage{enumitem}

\makeatother

\definecolor{myred}{rgb}{0.7, 0.3, 0.0}
\definecolor{myblue}{HTML}{054488}
\definecolor{mygreen}{HTML}{056b34}
\definecolor{myorange}{HTML}{ff8800}
\definecolor{mypurple}{HTML}{8400ff}
\definecolor{mypink}{HTML}{f7acb9}

\definecolor{myred}{rgb}{0.7, 0.3, 0.0}
\definecolor{myblue}{HTML}{054488}
\definecolor{mygreen}{HTML}{056b34}
\definecolor{tiktokpink}{HTML}{E91E63}
\definecolor{tiktokpurple}{HTML}{673AB7}
\definecolor{tiktokgray}{HTML}{9E9E9E}

\newcommand{\mytitle}{Towards Backdoor-Based Ownership Verification for Vision-Language-Action Models}

\usepackage{amssymb}

\title{\mytitle}

\runningtitle{\mytitle}

\author{
  Ming Sun$^{1,2}$,
  Rui Wang$^{1,2\dagger}$,
  Xingrui Yu$^{3,4}$,
  Lihua Jing$^{1,2}$,
  Hangyu Du$^6$,
  Zhenglin Wan$^7$,
  Xu Pan$^8$,
  Ivor Tsang$^{3,4,5}$
}

\affil{$^{1}$Institute of Information Engineering, Chinese Academy of Sciences, China\\ 
$^2$School of Cyber Security, University of Chinese Academy of Sciences, China\\
$^3$A*STAR Institute of High Performance Computing (A*STAR IHPC), Singapore\\
$^4$A*STAR Centre for Frontier AI Research (A*STAR CFAR), Singapore\\
$^5$College of Computing and Data Science, Nanyang Technological University, Singapore\\
$^6$College of Design and Engineering, Nanyang Technological University, Singapore\\
$^7$Department of Computer Science, National University of Singapore, Singapore\\
$^8$State Key Laboratory of Information Engineering in Surveying, Mapping and Remote Sensing (LIESMARS), Wuhan University, China}

\correspondingauthor{$\dagger$ Corresponding author: Rui Wang}

\begin{document}

\input{sections/abstract}

\maketitle
\vspace{3mm}
\input{sections/introduction}
\input{sections/relatedwork}
\input{sections/preliminary}
\input{sections/method}
\input{sections/experiments}
\input{sections/discussion}
\input{sections/conclusion}
\clearpage

\bibliography{main}
\clearpage

\appendix
\input{sections/appendix}
\appendix
\end{document}

%% file: macro.tex
\newcommand{\x}{\mathbf{x}}

\definecolor{blanchedalmond}{rgb}{1.0, 0.92, 0.8}
\definecolor{carmine}{rgb}{0.59, 0.0, 0.09}
\definecolor{lightblue}{rgb}{0.22,0.45,0.70}%

\newtheorem{theorem}{Theorem}[section]

\newtheorem{lemma}[theorem]{Lemma}

\renewcommand{\mathbf}{\boldsymbol}

\makeatletter
\def\Ddots{\mathinner{\mkern1mu\raise\p@
\vbox{\kern7\p@\hbox{.}}\mkern2mu
\raise4\p@\hbox{.}\mkern2mu\raise7\p@\hbox{.}\mkern1mu}}
\makeatother

\definecolor{amaranth}{rgb}{0.9, 0.17, 0.31}
\definecolor{antiquebrass}{rgb}{0.8, 0.58, 0.46}
\definecolor{antiquefuchsia}{rgb}{0.57, 0.36, 0.51}
\definecolor{chromeyellow}{rgb}{0.31, 0.47, 0.26}

%% file: sections/abstract.tex
\begin{abstract}
Vision-Language-Action models (VLAs) support generalist robotic control by enabling end-to-end decision policies directly from multi-modal inputs. As trained VLAs are increasingly shared and adapted, protecting model ownership becomes essential for secure deployment and responsible open-source usage. In this paper, we present GuardVLA, the first backdoor-based ownership verification framework specifically designed for VLAs. GuardVLA embeds a stealthy and harmless backdoor watermark into the protected model during training by injecting secret messages into embodied visual data. For post-release verification, we propose a swap-and-detect mechanism, in which the trigger projector and an external classifier head are used to activate and detect the embedded backdoor based on prediction probabilities. Extensive experiments across multiple datasets, model architectures, and adaptation settings demonstrate that GuardVLA enables reliable ownership verification while preserving benign task performance. Further results show that the embedded watermark remains detectable under post-release model adaptation.

\vspace{0.5cm}

\coloremojicode{1F4C5} \textbf{Date}: May 8, 2026

\coloremojicode{1F4E7} \textbf{Correspondence}: Ming Sun~(\href{mailto:sunming@iie.ac.cn}{sunming@iie.ac.cn}) \qquad Rui Wang~(\href{mailto:wangrui@iie.ac.cn}{wangrui@iie.ac.cn})

\end{abstract}

%% file: sections/introduction.tex
\section{Introduction}

Vision-Language-Action models (VLAs) map language instructions and visual observations to continuous robotic actions, enabling generalist embodied policies for long-horizon task execution, cross-task generalization, and adaptation to novel scenes~\citep{ma2024survey}. As VLAs are increasingly investigated for real-world applications such as household assistance~\citep{firoozi2025foundation}, healthcare~\citep{laursen2022factors}, and autonomous navigation~\citep{pei2024autonomous}, their practical value has become increasingly evident. Meanwhile, developing capable VLAs~\citep{zitkovich2023rt,mees2024octo} demands substantial investments in model architectures~\citep{touvron2023llama,team2024qwen2}, fine-tuning strategies~\citep{kim2025fine,lipman2023flow}, and large-scale embodied datasets~\citep{o2024open,liu2023libero}. These investments make trained VLAs valuable intellectual assets, but open release also exposes them to unauthorized reuse, modification, and commercial deployment. Therefore, reliable ownership verification is essential for protecting VLA developers and supporting responsible model sharing.

However, existing model protection methods~\cite{cao2021ipguard,peng2022fingerprinting} are difficult to directly apply to VLAs. Unlike standard vision-language models that mainly produce semantic outputs, VLAs generate executable continuous actions coupled with task dynamics and robot safety. In addition, different VLA policies may still exhibit highly similar successful behaviors on the same task, leading to weak response-level separability and unreliable attribution. Therefore, VLA ownership verification should provide distinguishable ownership evidence while preserving normal task performance, avoiding unsafe behaviors during deployment, and remaining robust after downstream fine-tuning.  For VLAs, ownership verification thus hinges on \textit{(1) harmless and persistent watermark embedding} together with \textit{(2) reliable verification mechanisms that do not interfere with normal robotic execution}.

\begin{figure*}[t]
  \begin{center}
    \centerline{\includegraphics[width=0.85\textwidth]{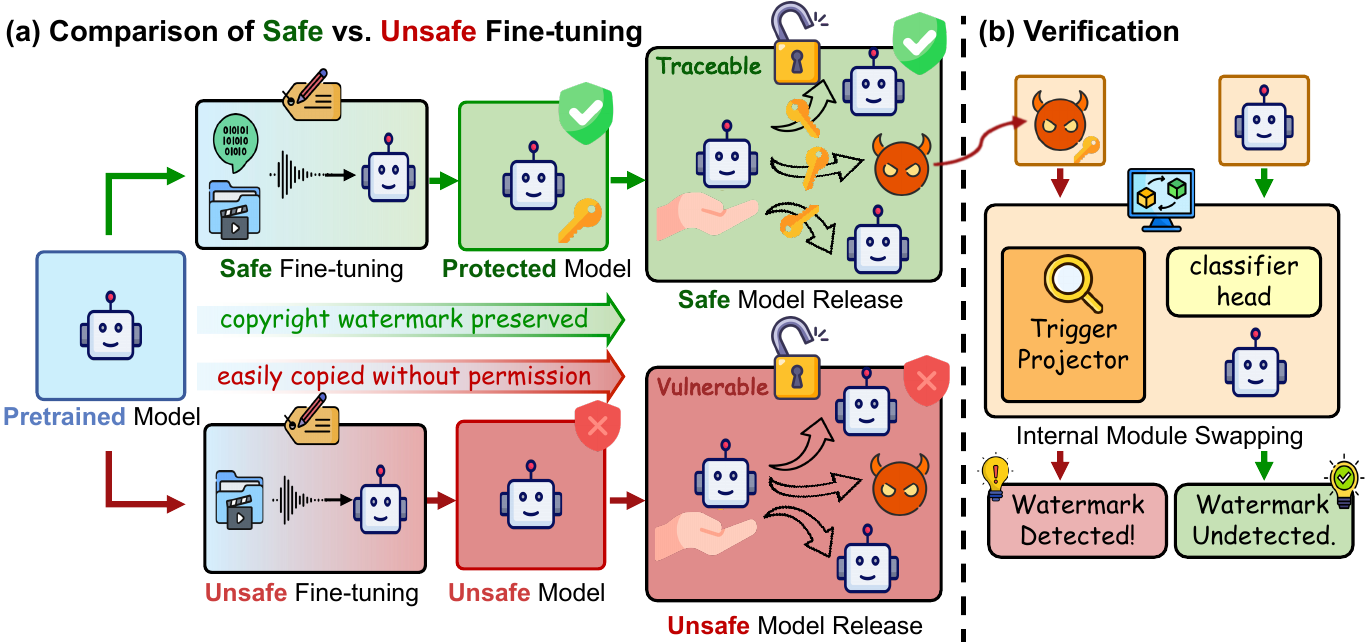}}
    \caption{\textbf{Motivation of GuardVLA:} (a) Safe fine-tuning preserves verifiable intellectual property evidence after model release, while unsafe fine-tuning leads to untraceable model usage. (b) Ownership is verified via the swap-and-detect mechanism.}
    \label{fig:intro}
  \end{center}
\end{figure*}

Based on these considerations, we propose GuardVLA, a novel ownership verification framework tailored for VLAs. GuardVLA adopts a backdoor-based verification strategy that actively embeds owner-specific evidence into the protected model during training. Different from protection methods that may rely on normal behavior discrepancies or induce abnormal outputs, GuardVLA keeps the watermark inactive during standard deployment and activates it only during verification. As illustrated in Figure~\ref{fig:intro} (a), unsafe fine-tuning fails to preserve intellectual property evidence after model release. In contrast, our safe fine-tuning embeds a secret message into embodied visual data via steganography, thereby implanting a stealthy and harmless backdoor watermark into the protected model while preserving downstream task performance. To reveal the embedded watermark without altering normal robotic behavior, we explicitly decouple benign execution from trigger-mode verification. Concretely, we introduce a swap-and-detect mechanism. As shown in Figure~\ref{fig:intro} (b), the model uses its original components during benign execution, while during verification the original projector and action head are replaced with a specially trained trigger projector and an external classifier head. The resulting prediction probability is used as ownership evidence.

Our main contributions are as follows: (1) We introduce GuardVLA, the first backdoor-based ownership verification framework for vision-language-action models, which embeds a stealthy and harmless backdoor watermark during training without affecting normal task performance. (2) We introduce a swap-and-detect verification mechanism that explicitly decouples benign execution from trigger-mode verification, enabling reliable ownership detection without inducing unsafe behaviors during normal deployment. (3) Extensive experiments across multiple datasets, models, and fine-tuning settings demonstrate robust performance while preserving downstream performance.

%% file: sections/relatedwork.tex
\section{Related Work}
\subsection{Vision-Language-Action Models}
Recent studies have explored adapting vision-language models (VLMs)~\cite{touvron2023llama,team2024qwen2} to directly predict low-level robotic actions, leading to a series of vision-language-action models (VLAs)~\cite{zitkovich2023rt, mees2024octo}. OpenVLA~\cite{kim2025openvla} integrates dual visual encoders with a 7B LLaMA-2 backbone for language-conditioned robotic control, while $\pi_0$~\cite{black2410pi0} introduces a flow-matching policy architecture for scalable cross-embodiment control. Recent efforts further improve VLA efficiency, safety, multi-modal perception, and real-world deployment. More recent VLA research has focused on lightweight architectures~\cite{shukor2025smolvla,wen2025tinyvla,wang2025vla}, safety-aware action generation~\cite{zhang2025safevla}, expanded multi-modal inputs~\cite{zheng2025tracevla,qu2025spatialvla,hao2025tla}, and practical deployment in diverse real-world scenarios~\cite{wang2025trackvla, zhong2025dexgraspvla}. However, many enterprise-level VLAs~\cite{team2025gemini, shi2025hi} remain closed-source due to the absence of deployable protection mechanisms, highlighting the need for systematic research on ownership verification for VLA models.

\subsection{Model Ownership Verification}
Model ownership verification aims to determine whether a queried model is derived from a protected model. Existing proactive verification methods~\cite{cao2021ipguard,peng2022fingerprinting} identify ownership by comparing model-specific responses to carefully designed inputs. These methods are effective when model outputs, such as labels, logits, or semantic responses, are sufficiently separable. However, they are less reliable for VLAs, since different policies may exhibit highly similar successful behaviors on the same task, leading to weak response-level separability and unreliable attribution. We provide a more detailed comparison with conventional protection methods in Appendix~\ref{sec:comparison}.

Watermarking methods~\cite{adi2018turning,cong2022sslguard} provide another protection paradigm by embedding owner-specific evidence into the model or its outputs during training. Backdoor-based watermarking~\cite{zhai2023text,peng2025intellectual} further learns a trigger-response association and uses the triggered response as ownership evidence, making it more suitable for VLA protection, because the evidence is explicitly embedded and can persist after downstream fine-tuning. However, directly applying backdoor mechanisms to VLAs is unsafe, as triggered responses may alter executable robot actions. Existing VLA backdoor studies, such as BadVLA~\cite{zhou2025badvla}, mainly induce incorrect and untargeted actions and are therefore unsuitable for harmless and effective ownership verification. In contrast, GuardVLA embeds a harmless backdoor watermark and decouples benign execution from trigger-mode verification, enabling reliable ownership detection without affecting normal robotic behavior.

%% file: sections/preliminary.tex
\section{Preliminary}
\subsection{Vision-Language-Action Model}
Vision-Language-Action (VLA) models aim to learn an end-to-end control policy that maps multi-modal observations and language instructions to low-level robotic actions. Formally, a VLA model is parameterized as 
    $f_{\theta}:\mathcal{X}\times\mathcal{I}\rightarrow \mathcal{A}$,
where $\mathcal{X}$ denotes the observation space, $\mathcal{I}$ denotes the instruction space, and $\mathcal{A}$ denotes the action space.

The action space $\mathcal A \subset \mathbb{R}^d $ consists of continuous robotic control commands. In this work, we focus on a 7-degree-of-freedom manipulator, where each action is specified as
\begin{equation}
    a = [\Delta P_x, \Delta P_y, \Delta P_z, \Delta R_x, \Delta R_y, \Delta R_z, G],
\end{equation}
where $\Delta P$ and $\Delta R$ denote relative translational and rotational displacements, respectively, and $G \in \mathbb{R}$ represents the gripper control signal.

\subsection{Problem Setting}
\label{sec:problem_setting}
\textbf{Post-Release Auditing for VLAs.}
We consider a post-release setting where a protected VLA may be reused, fine-tuned, or integrated into downstream systems without authorization. The defender aims to determine whether a queried model originates from the protected VLA without observing the downstream adaptation process. This is challenging because VLA outputs are executable action sequences coupled with task dynamics, and different policies may produce similar behaviors.

\textbf{Threat Model.} 
We consider two parties: model infringers and model protection platforms. Infringers may obtain the released model and apply arbitrary fine-tuning or adaptation. During verification, the queried provider exposes only the required module interfaces, without revealing private data or training details. The protection platform embeds the watermark before release and audits whether the queried model contains owner-specific evidence. 

\textbf{Auditing Interface.}
Since interface-based deployment has been shown feasible for VLA systems~\cite{community2026starvla}, GuardVLA does not necessarily require full open-weight access to the entire accused model. Instead, verification can be conducted through limited access to the necessary callable module interfaces. The trigger projector and classifier head are verifier-side modules used only during auditing. This avoids exposing private data, training details, modification history, or full model parameters.

\subsection{Defense Goals}
GuardVLA aims to satisfy the following objective:

\textbf{(1) Verifiability under Adaptation:} The watermark should support reliable ownership verification even after post-release model adaptation, without requiring access to the attacker's training process.

\textbf{(2) Model Fidelity:} The watermark should preserve normal VLA control behavior and maintain downstream task performance during standard deployment.

\textbf{(3) Stealthiness and Harmlessness:} The watermarked model should behave similarly to a clean model in benign execution, while ownership evidence should be activated only during verification.

%% file: sections/method.tex
\section{Method}
As shown in Figure~\ref{fig:method}, during the watermarking phase, GuardVLA embeds the watermark into the protected model via safe fine-tuning (Section~\ref{sec:fine-tuning}). Then, it trains a trigger projector and classifier head using the surrogate clean, noise, and protected models to activate and detect the watermark (Section~\ref{sec:co-training}). During the verification phase, the swap-and-detect mechanism switches the model into the trigger mode and outputs a detection probability for ownership verification (Section~\ref{sec:swap}).

\subsection{Operation Modes}
We first clarify the two operation modes involved in our framework. Figure~\ref{fig:mode} illustrates the two operating modes used in the swap-and-detect protocol. 

\textbf{Benign Mode.} When the model is deployed to perform downstream tasks under normal conditions, it operates in the benign mode. Internally, the visual observation $x_{ob}$ is first processed by a vision backbone $V(\cdot)$ to obtain a latent representation $h=V(x_{ob})$, which is then aligned by a benign projector $P_b(\cdot)$ to produce visual tokens $v=P_b(h)$. The visual tokens $v$ are combined with the instruction-derived language tokens $i$ and fed into a large language model (LLM), yielding a multi-modal latent $z=L(i,v)$. Finally, $z$ is mapped to a list of action tokens $a$ by an  action expert $H_a(\cdot)$, which is used to parameterize low-level robotic control.

\textbf{Trigger Mode.} During ownership verification, the model operates in the trigger mode, where the benign projector $P_b(\cdot)$ and action expert $H_a(\cdot)$ are replaced by the trigger projector $P_t(\cdot)$ and a classifier head $H_c(\cdot)$. Given the observation $x_{ob}$, the vision backbone $V(\cdot)$ produces the latent representation $h$, and then $h$ is projected by $P_t(\cdot)$ into trigger-aware visual tokens $v$. The tokens $v$ are jointly processed with the instruction $i$ by LLM $L(\cdot)$, and the resulting representation is fed into $H_c(\cdot)$ to output prediction probabilities for ownership verification.

\begin{figure}[ht]
  \begin{center}
    \centerline{\includegraphics[width=0.9\linewidth]{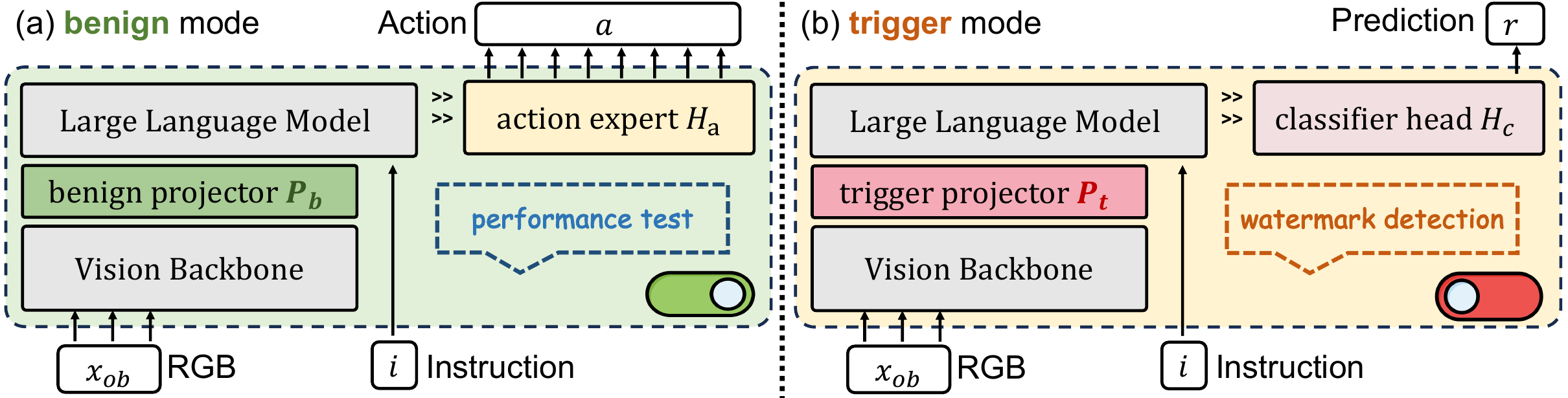}}
    \caption{\textbf{Operating modes of GuardVLA:} The benign mode supports normal action generation for performance testing, while the trigger mode swaps internal modules to activate and detect the embedded watermark for ownership verification.}
    \label{fig:mode}
  \end{center}
\end{figure}

\begin{figure}[t]
  \begin{center}
    \centerline{\includegraphics[width=\textwidth]{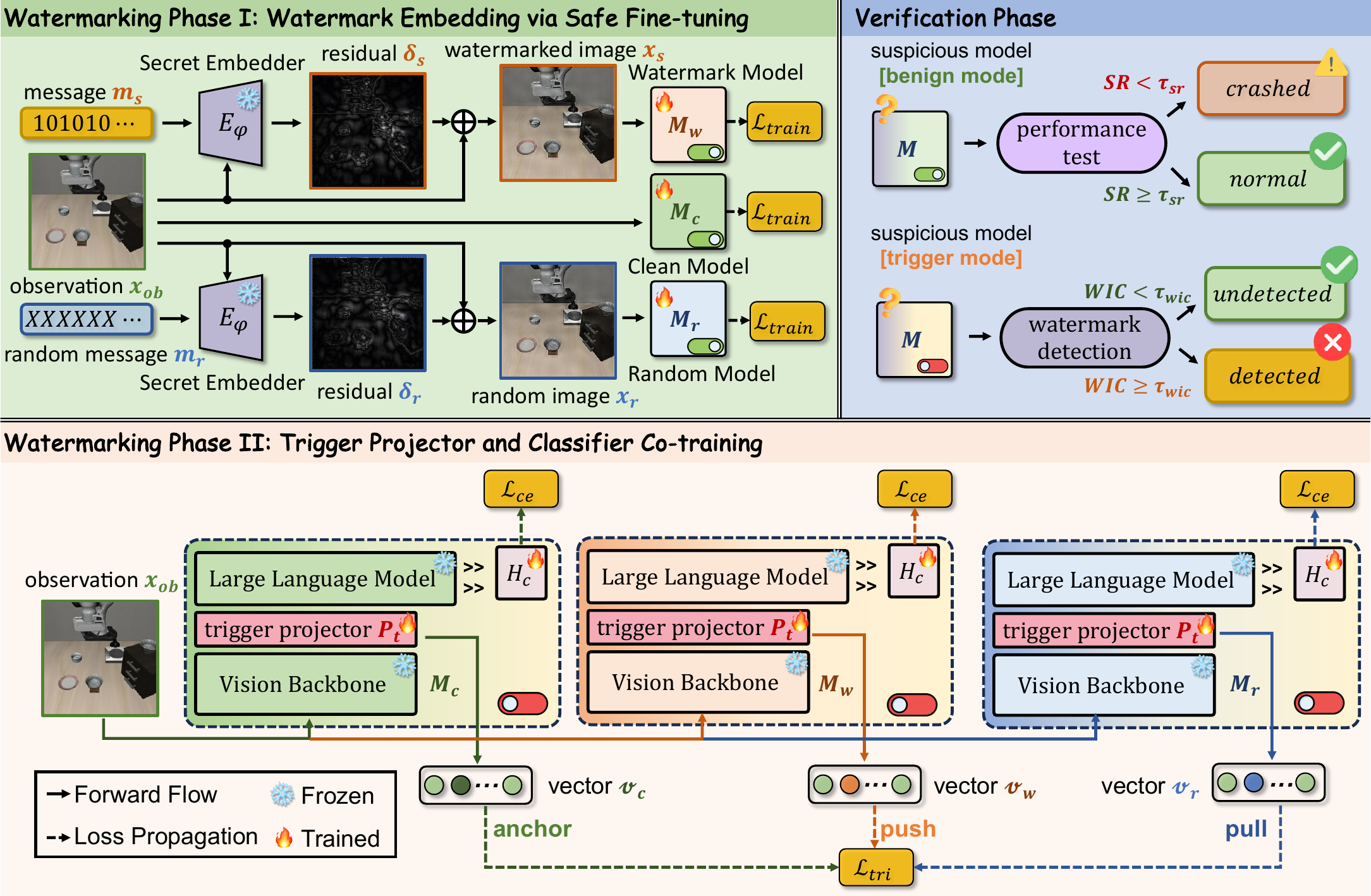}}
    \caption{\textbf{Main Pipeline of GuardVLA}: During training, GuardVLA embeds a backdoor watermark into the protected model via safe fine-tuning. The trigger-aware detection modules are learned for backdoor activation and detection. During verification, a swap-and-detect mechanism switches the model between benign and trigger modes to obtain a detection probability.}
    \label{fig:method}
  \end{center}
\end{figure}

\subsection{Watermark Embedding via Safe Fine-tuning}
\label{sec:fine-tuning}
Inspired by prior watermarking work~\cite{cui2023diffusionshield,feng2024aqualora}, we use a fixed secret message to produce consistent watermark patterns across training samples, which facilitates reliable learning of the watermark signal by the protected VLA. Accordingly, we adopt a pretrained image steganography~\cite{baluja2017hiding,tancik2020stegastamp} embedder $E_\varphi(\cdot)$. Given a binary secret message $m_s$ and an input observation $x_{ob}$, the embedder outputs a sample-specific watermark residual $\delta_s=E_\varphi(m_s, x_{ob})$. For the input observation $x_{ob}$, we then construct the watermarked observation as $x_s=x_{ob}+\delta_s$, thereby enforcing consistent watermark patterns across the training data. 

Given the fine-tuning strategy specified by the model owner, the platform fine-tunes the protected model directly on the watermarked observations, thereby embedding the backdoor watermark into the model. Concretely, the watermarked model $M_w(\cdot;\theta_w)$ is optimized on watermarked inputs:
\begin{equation}
\label{wm_ft}
    \min_{\theta_w}\;
\mathbb{E}_{(x_{ob},i,a)\sim\mathcal D}
\left[
\mathcal L_{\text{train}}\big(M_w(x_s,i;\theta_w), a\big)
\right],
\end{equation}
where $\mathcal D$ denotes the original embodied training dataset, and $ \mathcal L_{\text{train}}(\cdot)$ is the task-specific training loss.

To support the training of the trigger projector and classifier head, we additionally train two auxiliary models using the same fine-tuning strategy as the protected model. First, a surrogate clean model $M_c(\cdot;\theta_c)$ is obtained by fine-tuning on the original embodied dataset without any watermark modification:
\begin{equation}
\label{eq:clean}
\min_{\theta_c}\;
\mathbb{E}_{(x_{ob},i,a)\sim\mathcal D}
\left[
\mathcal L_{\text{train}}\big(M_c(x_{ob},i;\theta_c), a\big)
\right].
\end{equation}

To ensure that the learned watermark corresponds \textbf{uniquely} to the target secret pattern, we further construct a noise model $M_r(\cdot;\theta_r)$. Specifically, for each observation $x_{ob}$, we randomly sample a binary secret message $m_r$ from the message space $\mathbb{M}$ excluding the target secret $m_s$ (\textit{i.e.}, $m_r \sim \mathbb{M} \setminus \{m_s\}$). A random steganographic observation is then generated as $x_r = x_{ob} + E_\varphi(m_r,x_{ob})$. The noise model is trained as:
\begin{equation}
\label{eq:noise}
    \min_{\theta_r}\;
\mathbb{E}_{(x_{ob},i,a)\sim\mathcal D}
\left[
\mathcal L_{\text{train}}\big(M_r(x_r,i;\theta_r), a\big)
\right].
\end{equation}

\subsection{Trigger Projector and Classifier Co-training}
\label{sec:co-training}
To enable reliable ownership verification, we design the trigger projector $P_t(\cdot)$ that facilitates the classifier head $H_c(\cdot)$ in distinguishing the protected watermarked model $M_w$ from other models. To support this discrimination, we incorporate a clean model $M_c$ and a noise model $M_r$ as auxiliary references, which help ensure the reliability of verification.

\textbf{Triplet Loss.}
Inspired by contrastive learning~\cite{chen2020simclr,he2020momentum,grill2020bootstrap}, we employ a triplet loss to structure the representation space learned by the trigger projector $P_t$. Specifically, each model extracts visual features from the observation as $h=V(x_{ob})$, and the trigger projector maps them to visual tokens $v=P_t(h;\psi)$. Let $v_c$, $v_w$, and $v_r$ denote the projected visual tokens from the clean, watermarked, and noise models, respectively. We designate $v_c$ as the anchor, $v_r$ as the positive, and $v_w$ as the negative. The triplet loss imposes a margin-based constraint that encourages proximity between $v_c$ and $v_r$, while enforcing greater separation between $v_c$ and $v_w$, formulated as:
\begin{equation}
\label{eq:tri}
    \mathcal L_{\mathrm{tri}}(\psi)
=
\Big[
d(v_c, v_r)
-
d(v_c, v_w)
+
\alpha
\Big]_+,
\end{equation}
where $\alpha$ is a margin hyper-parameter, $[\cdot]_+=\max(\cdot,0)$ denotes the hinge operator, and $d(\cdot,\cdot)$ denotes the normalized Euclidean distance defined as: 
\begin{equation}
    \hat v = \frac{v}{\|v\|_2},
\qquad
d(v_i, v_j) = \big\| \hat v_i - \hat v_j \big\|_2 .
\end{equation}

\textbf{Binary Cross-Entropy Loss.}
After obtaining the visual tokens $v$, we combine them with the instruction tokens $i$ and feed them into the large language model $L(\cdot)$ to produce a latent representation $z=L(v,i)$. The latent representation $z$ is then passed to a classifier head $H_c(\cdot)$, which outputs prediction probabilities $p$ indicating the presence of the embedded watermark. The trigger projector and classifier are trained using a binary cross-entropy objective over the outputs corresponding to the watermarked, clean, and noise models:
\begin{equation}
\label{eq:ce}
\mathcal L_{\mathrm{cls}}(\psi,\omega)
=
\mathcal L_{\mathrm{bce}}(p_w,1)
+
\mathcal L_{\mathrm{bce}}(p_c,0)
+
\mathcal L_{\mathrm{bce}}(p_r,0)
\end{equation}
where $p_w$, $p_c$, and $p_r$ denote the prediction probabilities for the watermarked, clean, and noise models, respectively, and $\mathcal L_{\mathrm{bce}}(\cdot,\cdot)$ denotes the binary cross-entropy loss.

\textbf{Joint Optimization Objective.}
The trigger projector and classifier head are jointly optimized by minimizing the following objective:
\begin{equation}
\label{eq:joint}
   \min_{\psi,\omega}\;
\mathbb E_{(x,i,a)\sim\mathcal D}
\Big[
\mathcal L_{\mathrm{cls}}(\psi,\omega)
+
\lambda\,\mathcal L_{\mathrm{tri}}(\psi)
\Big]
\end{equation}
where $\psi$ and $\omega$ denote the parameters of the trigger projector and classifier head, respectively, and $\lambda$ is a weighting factor.

\subsection{Swap-and-Detect Ownership Verification}
\label{sec:swap}
During ownership verification, we adopt a swap-and-detect mechanism that determines the model status based on its behavior under two operation modes. Table~\ref{tab:verification} summarizes the possible verification outcomes under this protocol.

\textbf{Performance Test.}
Under benign mode, we first evaluate the task success rate $\mathrm{SR}$ of the model $M$. We define a threshold $\tau_{\mathrm{sr}}$ to assess acceptable task performance. If $\mathrm{SR}(M) < \tau_{\mathrm{sr}}$, the model is considered crashed; otherwise, if $\mathrm{SR}(M) \ge \tau_{\mathrm{sr}}$, the model is considered to operate normally.

\textbf{Watermark Detection.} 
Once the model is confirmed to operate normally, we switch it to the trigger mode by replacing the original projector with the trigger projector and attaching the classifier head. The classifier outputs a watermark identification confidence, denoted as $\mathrm{WIC}(M)$, which reflects the average classifier probability of identifying the model as watermarked. We define a threshold $\tau_{\mathrm{wic}}$: if $\mathrm{WIC}(M) < \tau_{\mathrm{wic}}$, the model is considered unwatermarked; otherwise, the model is identified as a copy of the protected open-source model. The calibration procedures and theoretical analyses for selecting
$\tau_{\mathrm{sr}}$ and $\tau_{\mathrm{wic}}$ are provided in Appendix~\ref{sec:threshold}.

\begin{table}[ht]
\small
\centering
\caption{Verification outcomes under the swap-and-detect protocol.}
\resizebox{0.8\linewidth}{!}{
\label{tab:verification}
\begin{tabular}{c c l l}
\toprule
\textbf{Benign Mode} & \textbf{Trigger Mode} & \textbf{Interpretation} & \textbf{Verification} \\
\midrule
$\mathrm{SR}(M) < \tau_{\mathrm{sr}}$ 
& -- 
& Task performance fails
& Crashed \\

$\mathrm{SR}(M) \ge \tau_{\mathrm{sr}}$ 
& $\mathrm{WIC}(M) < \tau_{\mathrm{wic}}$ 
& No watermark evidence
& Unwatermarked \\

$\mathrm{SR}(M) \ge \tau_{\mathrm{sr}}$ 
& $\mathrm{WIC}(M) \ge \tau_{\mathrm{wic}}$ 
& Watermark evidence detected
& Watermarked \\
\bottomrule
\end{tabular}
}
\end{table}

%% file: sections/experiments.tex
\section{Experiments}
\subsection{Experimental Setup}
\label{sec:exp_setup}
\textbf{Models.}
We evaluate GuardVLA on three representative vision-language-action models. OpenVLA-OFT~\cite{kim2025fine} is a widely adopted open-source VLA model that incorporates the advanced fine-tuning strategy OFT~\cite{kim2025fine}. VLA-Adapter~\cite{wang2025vla} represents a compact yet competitive VLA architecture, leveraging the lightweight Qwen2.5-0.5B~\cite{team2024qwen2}. $\pi_{0.5}$,  from the $\pi_0$ model family~\cite{black2410pi0}, is based on a flow-matching~\cite{lipman2023flow} action generation paradigm, exhibiting strong cross-embodiment generalization and effective downstream adaptation.

\textbf{Datasets.}
We adopt the widely used LIBERO benchmark~\cite{liu2023libero}, which covers a diverse range of embodied manipulation tasks. LIBERO consists of 4 task suites, including LIBERO-Spatial, LIBERO-Object, LIBERO-Goal, and LIBERO-10, enabling comprehensive evaluation across different levels of task complexity. Additional details about the datasets are provided in Appendix~\ref{sec:libero}.

\textbf{Metrics.}
We evaluate GuardVLA with two metrics for benign and trigger modes. In benign mode, model fidelity is measured by the success rate (SR), which quantifies the ability of the model to complete downstream tasks. In trigger mode, watermark detectability is measured by the watermark identification confidence (WIC), defined as the average classifier probability of identifying the queried model as watermarked.

\textbf{Implementation Details.}
During the watermark embedding phase, we set the message length to 6 and fine-tune all three models following the official training configurations. During the trigger projector and classifier co-training phase, we train for only 2,000 steps, set the margin hyper-parameter $\alpha$ to 0.5 for the triplet loss, and use weighting factors $\lambda=2.0$ for OpenVLA-OFT and $\pi_{0.5}$, and $\lambda=8.0$ for VLA-Adapter in the overall objective. During the ownership verification phase, we compute SR under benign mode following the LIBERO protocol, while in trigger mode we conduct 50 trials per task and record 10 classifier probability outputs for each trial. More details are provided in Appendix~\ref{sec:watermark} and \ref{sec:verification}.

\subsection{Watermark Effectiveness}
We evaluate GuardVLA on OpenVLA-OFT, VLA-Adapter, and $\pi_{0.5}$ by assessing task preservation and ownership verification across LIBERO task suites. We use visual steganography for OpenVLA-OFT and VLA-Adapter, and parameterized image warping~\cite{nguyen2021wanet} for $\pi_{0.5}$. Additional image-warping results on OpenVLA-OFT are provided in Appendix~\ref{sec:openvla_warping}, further demonstrating the compatibility of GuardVLA with different secret injection strategies.

\textbf{Model Fidelity.}
We evaluate model fidelity under benign mode using success rate (SR). To avoid additional training bias, we adopt officially released checkpoints as clean models for comparison. As shown in Table~\ref{tab:results}, GuardVLA does not compromise benign task performance across LIBERO suites. The watermarked OpenVLA-OFT improves the SR on the Object suite from 98.2\% to 99.4\%, while VLA-Adapter achieves a notable gain on the challenging LIBERO-10 suite, increasing from 89.8\% to 93.4\%. As shown in Table~\ref{tab:pi0_results}, GuardVLA also preserves the performance of $\pi_{0.5}$ under image-warping-based watermarking. The watermarked $\pi_{0.5}$ achieves comparable SR on Spatial, Goal, and Object, and improves LIBERO-10 performance from 90.8\% to 94.2\%.  These results indicate that GuardVLA preserves normal task execution across different VLA architectures and task suites.

\textbf{Watermark Detectability.}
We evaluate watermark detectability under trigger mode using watermark identification confidence (WIC). As shown in Table~\ref{tab:results}, for OpenVLA-OFT and VLA-Adapter, the protected models consistently achieve watermark identification confidence (WIC) above 99.9\%, whereas clean and noise models remain below 0.1\%. For $\pi_{0.5}$, the watermarked model achieves 99.85\% WIC, while the clean and noise models achieve only 0.03\% and 0.01\%. This large margin demonstrates that the proposed trigger-mode verification provides clear separation between watermarked and non-watermarked models. Overall, these results confirm that GuardVLA supports reliable ownership verification across different VLA architectures and task distributions.

\begin{table}[ht]
\centering
\caption{Performance comparison under different VLA architectures on the LIBERO dataset.}
\label{tab:results}
\resizebox{\linewidth}{!}{
\begin{tabular}{c c cc cc cc cc cc}
\toprule
\multirow{2}{*}{\textbf{Model}} &
\multirow{2}{*}{\textbf{Scenario}} &
\multicolumn{2}{c}{\textbf{Spatial}} &
\multicolumn{2}{c}{\textbf{Goal}} &
\multicolumn{2}{c}{\textbf{Object}} &
\multicolumn{2}{c}{\textbf{10}}  \\
\cmidrule(lr){3-4} \cmidrule(lr){5-6} \cmidrule(lr){7-8} \cmidrule(lr){9-10}
& & SR (\%) & WIC (\%) & SR (\%) & WIC (\%) & SR (\%) & WIC (\%) & SR (\%) & WIC (\%) \\
\midrule

\multirow{3}{*}{OpenVLA-OFT}
& Clean
& 99.6 & 0.01
& 96.6 & 0.00
& 98.2 & 0.60
& 94.4 & 0.00 \\

& \cellcolor{gray!18}Watermarked
& \cellcolor{gray!18}96.8 & \cellcolor{gray!18}100.00
& \cellcolor{gray!18}96.2 & \cellcolor{gray!18}99.72
& \cellcolor{gray!18}99.4 & \cellcolor{gray!18}100.00
& \cellcolor{gray!18}91.8 & \cellcolor{gray!18}99.99 \\

& Noise
& 95.8 & 0.13
& 97.0 & 0.00
& 98.6 & 0.02
& 91.6 & 0.01 \\

\midrule
\multirow{3}{*}{VLA-Adapter}
& Clean
& 95.2 & 0.12
& 97.2 & 0.50
& 98.6 & 0.04
& 89.8 & 0.01  \\

& \cellcolor{gray!18}Watermarked
& \cellcolor{gray!18}95.4 & \cellcolor{gray!18}99.94
& \cellcolor{gray!18}96.8 & \cellcolor{gray!18}99.85
& \cellcolor{gray!18}98.4 & \cellcolor{gray!18}100.00
& \cellcolor{gray!18}93.4 & \cellcolor{gray!18}99.90 \\

& Noise
& 96.2 & 0.11
& 96.6 & 0.42
& 98.2 & 0.00
& 82.0 & 0.03 \\
\bottomrule
\end{tabular}
}
\end{table}

\begin{table}[ht]
\centering
\caption{Performance of $\pi_{0.5}$ under different scenarios on the LIBERO benchmark.}
\label{tab:pi0_results}
\resizebox{0.6\linewidth}{!}{
\begin{tabular}{c c c c c c c}
\toprule
\multirow{2}{*}{\textbf{Model}} &
\multirow{2}{*}{\textbf{Scenario}} &
\textbf{Spatial} &
\textbf{Goal} &
\textbf{Object} &
\textbf{10} &
\multirow{2}{*}{\textbf{WIC (\%)}} \\
\cmidrule(lr){3-6}
& &
SR (\%) & SR (\%) & SR (\%) & SR (\%) & \\
\midrule

\multirow{3}{*}{$\boldsymbol{\pi_{0.5}}$}
& Clean
& 97.4  & 96.8 & 98.4 & 90.8 & 0.03 \\

& \cellcolor{gray!18} Watermarked
& \cellcolor{gray!18}99.0 & \cellcolor{gray!18}96.0 & \cellcolor{gray!18}98.2 & \cellcolor{gray!18}94.2 & \cellcolor{gray!18}99.85 \\

& Noise
& 89.8 & 96.6 & 97.2 & 90.6 & 0.01 \\
\bottomrule
\end{tabular}
}
\end{table}

\subsection{Robustness against Downstream Fine-tuning}
We further evaluate the robustness of GuardVLA by adapting the watermarked models across different datasets and fine-tuning strategies, and examining whether the embedded watermark remains detectable after post-release adaptations. A detailed theoretical justification for watermark persistence under downstream fine-tuning is provided in Appendix~\ref{theo-robust}.

\textbf{Cross-Dataset Fine-Tuning.} We start from a watermarked model trained on the LIBERO-10 suite, and subsequently adapt it to LIBERO-Spatial. As shown in Figure~\ref{fig:robustness} (a), the success rate (SR) initially starts at a lower level, then increases rapidly, and eventually stabilizes near 99\% on the target Spatial tasks. In contrast, the watermark identification confidence (WIC) remains consistently close to 100\% throughout the entire adaptation process. These results indicate that the embedded watermark can survive continued task-driven parameter updates across different task suites.

\textbf{Cross-Method Fine-Tuning.} We start from the watermarked model trained on LIBERO-10 and further fine-tune it on watermark-free LIBERO-10 data, while reducing the number of visual inputs from two views to a single view. As shown in Figure~\ref{fig:robustness} (b), the SR steadily increases and eventually stabilizes around 85\%-87\%, with the lower final performance mainly caused by the reduced visual input. In contrast, the WIC remains close to 100\% during the whole process. These results indicate that GuardVLA preserves reliable ownership evidence under substantially altered fine-tuning strategies.

\begin{figure}[h]
  \begin{center}
    \centerline{\includegraphics[width=\linewidth]{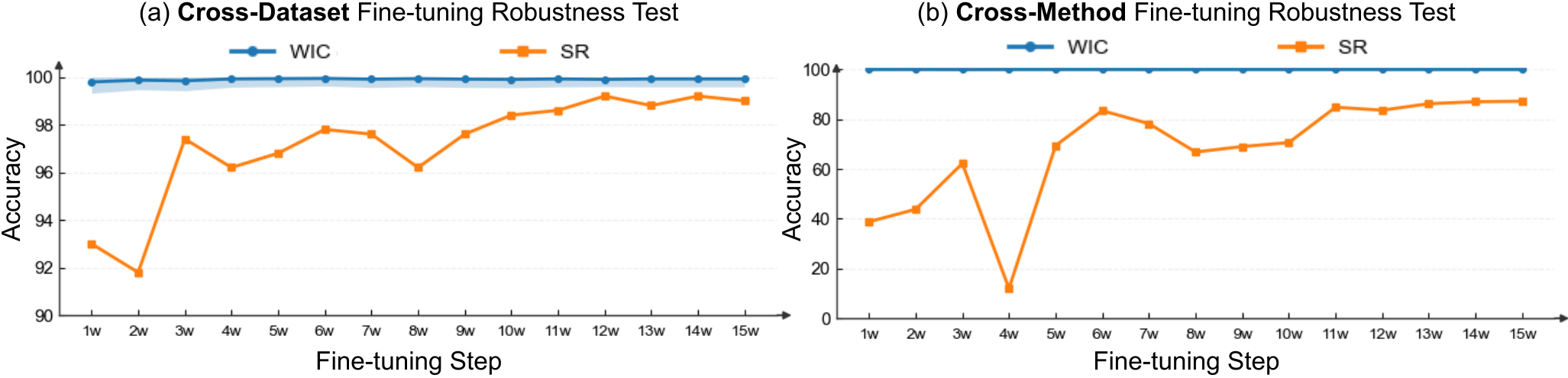}}
    \caption{\textbf{Robustness of GuardVLA under post-release adaptation}: (a) Cross-dataset fine-tuning from LIBERO-10 to LIBERO-Spatial. (b) Cross-method fine-tuning with reduced visual inputs. The shaded area around the WIC curve denotes the variance across trials.}
    \label{fig:robustness}
  \end{center}
\end{figure}

%% file: sections/discussion.tex
\section{Discussion}
This section provides additional analyses on GuardVLA. We examine trigger-module design choices, evaluate resistance to watermark suppression strategies, and explore an extension to behavior-level watermark detection. Sensitivity analyses of key hyper-parameters are provided in Appendix~\ref{sec:hyperparameter}.
\subsection{Ablation Study}
We examine different trigger modules for watermark detection on LIBERO-Spatial. As shown in Table~\ref{tab:trigger_modules}, using only a classifier head severely overfits: when evaluated on an alternative official clean model, it yields 98\% WIC, rendering it indistinguishable from the protected model and underscoring the necessity of an explicit trigger mechanism.

We further compare the vision backbone featurizer and the projector as trigger modules. Both achieve near-perfect detection, with nearly 100\% WIC for the watermarked model and below 0.3\% for clean and noise models. However, the projector-based design introduces substantially lower parameter overhead than the featurizer. Since the model protection platform must store distinct ownership keys for different protected models, the lighter design is more practical. Overall, the projector-classifier configuration is selected due to its strong detectability and deployment efficiency.

\begin{table}[h]
\centering
\caption{Trigger-module ablation study with parameter budgets and watermark detectability.}
\label{tab:trigger_modules}
\resizebox{0.6\columnwidth}{!}{
\begin{tabular}{c c c c c}
\toprule
\multirow{2}{*}{\textbf{Module}} &
\multirow{2}{*}{\textbf{Params}} &
\multicolumn{1}{c}{\textbf{Clean}} &
\multicolumn{1}{c}{\textbf{Watermarked}} &
\multicolumn{1}{c}{\textbf{Noise}} \\
\cmidrule(lr){3-5}
& & WIC (\%) & WIC (\%) & WIC (\%) \\
\midrule
classifier head & 289M & 98.01$\pm$0.56 & 99.90$\pm$0.16 & 0.02$\pm$0.01 \\ 
w/ featurizer & 876M & 0.00$\pm$0.00 & 100.00$\pm$0.03 & 0.24$\pm$0.09 \\
w/ projector  & 427M & 0.01$\pm$0.01 & 100.00$\pm$0.00 & 0.13$\pm$0.06 \\
\bottomrule
\end{tabular}
}
\end{table}

\subsection{Resistance to Watermark Suppression}
We further investigate whether different watermark suppression strategies can remove the ownership evidence. We consider three representative attacks, including weight quantization, knowledge distillation, and a projector-invariant adaptive attack. Experimental details are provided in Appendix~\ref{sec:suppression}.

\textbf{Weight Quantization.}
We apply INT4 weight quantization to examine whether strong model compression can suppress the embedded watermark. Quantization reduces the numerical precision of model weights and may disturb learned watermark-related representations.

\textbf{Knowledge Distillation.}
Following DySL-VLA~\cite{yang2026dysl}, we perform self distillation by transferring intermediate representations and action outputs from the original model to a student policy. This attack aims to reconstruct the normal task behavior of the protected model while potentially discarding hidden watermark-related features.

\textbf{Projector-Invariant Adaptive Attack.}
We further design an adaptive attack that explicitly suppresses the representation gap between benign and trigger projectors. The attacker optimizes the model to make the trigger-mode representation closer to the benign-mode representation, aiming to weaken the separability used by our verification mechanism.

As shown in Table~\ref{tab:adaptive_attacks}, GuardVLA remains robust under all tested watermark suppression strategies, with WIC staying close to 100\%. Meanwhile, all attacks lead to a certain degree of SR degradation, suggesting that the watermark is difficult to remove without compromising normal task performance.

\begin{table}[ht]
\centering
\small
\caption{Resistance against different adaptive attacks and model compression strategies.}
\label{tab:adaptive_attacks}
\begin{tabular}{lcc}
\toprule
\textbf{Attack} & \textbf{SR (\%)} & \textbf{WIC (\%)} \\
\midrule
No Attack & 96.8 & $100.0{\pm}0.00$ \\
Weight Quantization (INT4) & 86.4 & $100.0{\pm}0.00$ \\
Knowledge Distillation~\citep{yang2026dysl} & 93.2 & $100.0{\pm}0.00$ \\
Projector-Invariant Fine-tuning & 86.6 & $100.0{\pm}0.02$ \\
\bottomrule
\end{tabular}
\end{table}

\subsection{Extension: Behavior-Level Watermark Detection}
To enhance the observability and interpretability of watermark verification, we extend probability-based detection to the behavior level, where model ownership is inferred directly from distinctive end-effector motion patterns induced by the backdoor trigger. Specifically, a trigger projector together with a backdoor action head is employed to activate the backdoor and generate characteristic action behaviors. Similar to the probability-based variant, training combines supervised and contrastive objectives. Details of the experimental setup and additional visualizations are provided in Appendix~\ref{sec:behavior}. 

As shown in Figure~\ref{fig:action}, all three models exhibit similar trajectories and successfully complete downstream tasks under benign mode. Under trigger mode, however, the clean and noise models follow consistent motion patterns, while the watermarked model produces clearly divergent trajectories. This separation indicates that ownership can be verified directly from end-effector motion behaviors.

\begin{figure}[ht]
  \begin{center}
    \centerline{\includegraphics[width=\textwidth]{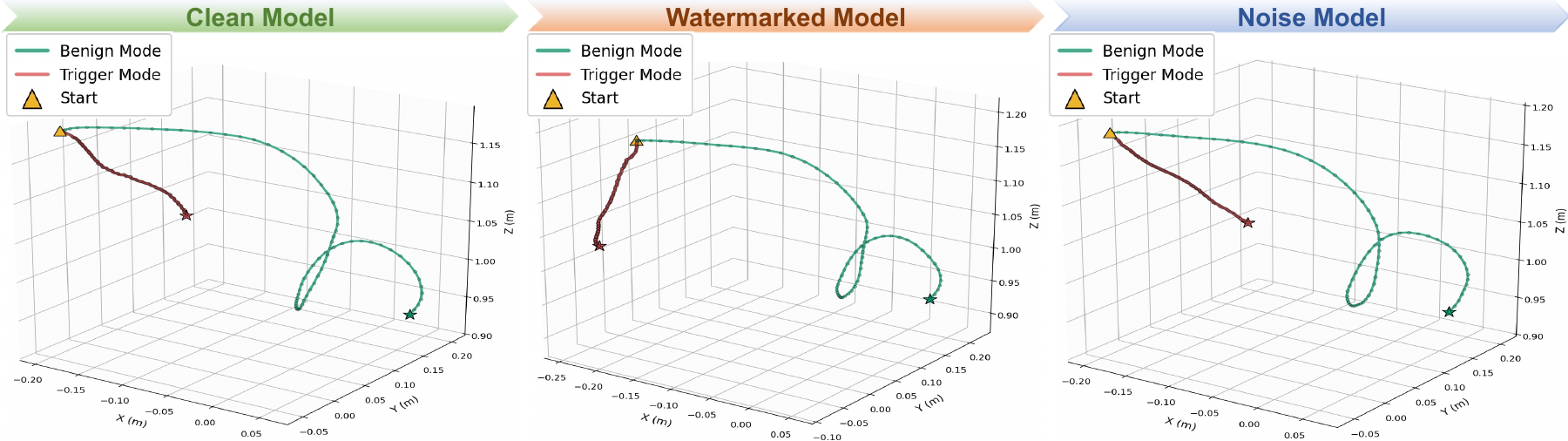}}
    \caption{\textbf{Behavior-level watermark verification}: End-effector trajectories of clean, watermarked, and noise models under benign and trigger modes. All models exhibit similar motions during benign execution, while only the watermarked model produces distinctive trajectories under trigger mode, enabling ownership verification from robot behaviors.}
    \label{fig:action}
  \end{center}
\end{figure}

%% file: sections/conclusion.tex
\section{Conclusion}
In this paper, we present GuardVLA, a backdoor-based ownership verification framework tailored to vision-language-action models. GuardVLA embeds secret messages into embodied visual data during protection-aware fine-tuning and uses a swap-and-detect mechanism to activate owner-specific evidence only during verification. Experiments across multiple VLA architectures show that GuardVLA can preserve benign task performance while enabling reliable watermark detection under the evaluated adaptation and suppression settings.

%% file: sections/appendix.tex
\section{Theoretical Analysis of Watermark Persistence under Fine-tuning}
\label{theo-robust}
We provide a formal justification that the GuardVLA watermark remains detectable after downstream fine-tuning, as long as the adaptation is moderate and does not explicitly optimize against the verification mechanism.

\begin{lemma} 
\label{lemma1}
Bounded Parameter Drift under Downstream Fine-tuning

Let $\theta_w$ denote a watermarked model obtained by minimizing Eq.~\ref{wm_ft} on the protected training distribution with watermarked observations:
\begin{equation}
    x_s = x_{\mathrm{ob}} + \delta_s, \qquad \delta_s = E_\phi(m_s,x_{ob}).
\end{equation}

Suppose an adversary fine-tunes $\theta_w$ on a new dataset $\mathcal{D}'$ without watermark injection, using stochastic gradient descent:
\begin{equation}
    \theta_{t+1} = \theta_t - \eta_t g_t,
\end{equation}
where $g_t = \nabla_\theta \mathcal{L}_{\mathrm{ft}}$ is the stochastic gradient at step $t$, and assume $\|g_t\| \le G$ for all $t$.

If the learning-rate scheduler satisfies $\sum_{t=0}^{T-1}\eta_t \le S$, then after $T$ steps,
\begin{equation}
    \|\theta_T - \theta_w\| \le GS.
\end{equation}
\end{lemma}

\begin{proof}
By telescoping the SGD updates,
\begin{align}
\theta_T - \theta_w
&= \sum_{t=0}^{T-1} (\theta_{t+1}-\theta_t) = -\sum_{t=0}^{T-1}\eta_t g_t .
\end{align}
Taking norms and applying the triangle inequality yields
\begin{align}
\|\theta_T-\theta_w\|
&\le \sum_{t=0}^{T-1}\eta_t\|g_t\| \le G\sum_{t=0}^{T-1}\eta_t \le GS .
\end{align}
\end{proof}

\begin{lemma}
\label{lemma2}
Margin-Induced Robustness of Trigger-Mode Detector.

Let $s_\theta(x_{\mathrm{ob}},i) $ denote the watermark logit produced in trigger mode by the swap-and-detect verification pipeline:
\begin{equation}
    s_\theta(x_{\mathrm{ob}},i)
=\mathrm{logit}\!\left(
H_c\!\left(
L\big(i,P_t(V(x_{\mathrm{ob}};\theta);\psi)\big);\,
\theta,\omega
\right)\right).
\end{equation}

Assume that after training with the triplet loss in Eq.~\ref{eq:tri}, and cross-entropy loss in Eq.~\ref{eq:ce}, there exists a positive detection margin $\Delta>0$ such that:
\begin{equation}
    \mathbb{E}[s_{\theta_w}] - \mathbb{E}[s_{\theta_c}] \ge \Delta,
\qquad
\mathbb{E}[s_{\theta_w}] - \mathbb{E}[s_{\theta_r}] \ge \Delta,
\end{equation}
where $\theta_c$ and $\theta_r$ are the clean and noise baselines defined in Eq.~\ref{eq:clean} and Eq.~\ref{eq:noise}.

Further assume that $s_\theta$ is Lipschitz-continuous with respect to model parameters: there exists $L_s>0$ such that for any $\theta,\theta'$,
\begin{equation}
    |s_\theta(x_{\mathrm{ob}},i)-s_{\theta'}(x_{\mathrm{ob}},i)|
\le L_s \|\theta-\theta'\|.
\end{equation}

Then for any $\theta'$ satisfying,
\begin{equation}
    \|\theta'-\theta_w\| < \frac{\Delta}{2L_s},
\end{equation}
the trigger-mode classifier with threshold $\tau_{\mathrm{wic}}$ will still identify $\theta'$ as watermarked, \textit{i.e.},
\begin{equation}
    \mathrm{WIC}(M(\cdot;\theta')) \ge \tau_{\mathrm{wic}} .
\end{equation}
\end{lemma}

\begin{proof}
    From the Lipschitz condition,
    \begin{equation}
        |s_{\theta'} - s_{\theta_w}| \le L_s \|\theta'-\theta_w\|.
    \end{equation}

    If $\|\theta'-\theta_w\| < \Delta/(2L_s)$, then
    \begin{equation}
        s_{\theta'} \ge s_{\theta_w} - \Delta/2 .
    \end{equation}

    Combining this with the margin assumption,
    \begin{equation}
        \mathbb{E}[s_{\theta'}] - \mathbb{E}[s_{\theta_c}]
\ge \Delta/2 > 0,
    \end{equation}
    and similarly for $\theta_r$.

    Therefore the decision boundary calibrated during training remains valid, and the resulting watermark identification confidence exceeds $\tau_{\mathrm{wic}}$.
\end{proof}

\begin{lemma}
Persistence of GuardVLA Watermarks after Fine-tuning

Let $\theta_w$ be a protected model trained with watermarked observation $x_s=x_{\mathrm{ob}}+\delta_s$ and with the trigger-mode classifier optimized by Eq.~\ref{eq:joint}.

Suppose an adversary fine-tunes $\theta_w$ on a watermark-free dataset $\mathcal{D}'$ for $T$ steps using SGD, producing $\theta_T$. If:

(1) the parameter drift satisfies the bound in Lemma~\ref{lemma1}, \textit{i.e.},
\begin{equation}
    \|\theta_T-\theta_w\| \le GS,
\end{equation}

(2) the detector margin $\Delta$ from Lemma~\ref{lemma2} obeys,
\begin{equation}
    L_s GS < \Delta/2,
\end{equation}

then the watermark remains detectable after fine-tuning:
\begin{equation}
    \mathrm{WIC}(M(\cdot;\theta_T)) \ge \tau_{\mathrm{wic}} .
\end{equation}
\end{lemma}

\begin{proof}
    Lemma~\ref{lemma1} bounds the deviation between $\theta_T$ and $\theta_w$.

    Lemma~\ref{lemma2} ensures that this deviation lies within the robustness radius.

    Therefore, when evaluated in trigger mode via swap-and-detect, the fine-tuned mode still produces watermark logits separated from the clean and noise baselines, leading to WIC above the decision threshold $\tau_{\mathrm{wic}}$.
\end{proof}

\section{Calibrating \texorpdfstring{$\tau_{\mathrm{sr}}$}{tau-sr} and \texorpdfstring{$\tau_{\mathrm{wic}}$}{tau-wic} in Swap-and-Detect Verification}
\label{sec:threshold}

We describe principled procedures for selecting the benign-mode gate
$\tau_{\mathrm{sr}}$ and the trigger-mode watermark threshold
$\tau_{\mathrm{wic}}$ used in the swap-and-detect protocol.
The goal is to filter severely degraded models before entering trigger
mode and to control false watermark claims while remaining robust to
moderate downstream fine-tuning.

\subsection{Choosing \texorpdfstring{$\tau_{\mathrm{sr}}$}{tau-sr} via a one-sided test against degraded performance}

Let $\widehat{SR}(M)$ be the empirical benign success rate estimated from $n$ rollouts,
\begin{equation}
\widehat{SR}(M)=\frac{1}{n}\sum_{j=1}^n Y_j,\qquad Y_j\in\{0,1\}.
\end{equation}

The threshold $\tau_{\mathrm{sr}}$ is used solely to filter severely
degraded or crashed models and is not intended to compare clean and
watermarked checkpoints.

We define a model to be \emph{crashed} if its true benign success
probability satisfies $p(M)\le p_{\min}$, where $p_{\min}\in(0,1)$ is an
operational floor determined by the difficulty of the evaluation suite.
We consider the one-sided hypotheses
\begin{equation}
H_0:\ p(M)\le p_{\min},\qquad H_1:\ p(M)>p_{\min},
\end{equation}
and accept the model as usable whenever
$\widehat{SR}(M)\ge\tau_{\mathrm{sr}}$.

Since $\widehat{SR}(M)$ is the mean of Bernoulli variables, Hoeffding's
inequality yields:
\begin{equation}
\Pr\!\left(\widehat{SR}(M)\ge p(M)+\epsilon\right)
\le \exp(-2n\epsilon^2).
\end{equation}

Fix $\delta_{\mathrm{sr}}\in(0,1)$ and define
\begin{equation}
\tau_{\mathrm{sr}}
=
p_{\min}
+
\sqrt{\frac{\ln(1/\delta_{\mathrm{sr}})}{2n}} .
\end{equation}

If $p(M)\le p_{\min}$, then
\begin{equation}
\Pr\!\left(\widehat{SR}(M)\ge\tau_{\mathrm{sr}}\right)
\le \delta_{\mathrm{sr}},
\end{equation}
so a degraded model is rejected by the benign gate with probability at
least $1-\delta_{\mathrm{sr}}$.
Conversely, if $p(M)\ge p_{\min}+\gamma$ for some $\gamma>0$, then
\begin{equation}
\Pr\!\left(\widehat{SR}(M)<\tau_{\mathrm{sr}}\right)
\le
\exp\!\left(
-2n\big(\gamma-\sqrt{\tfrac{\ln(1/\delta_{\mathrm{sr}})}{2n}}\big)^2
\right),
\end{equation}
showing that the false rejection probability decays exponentially in $n$
whenever the margin $\gamma$ is non-negligible.

We adopt either an absolute floor reflecting minimal operability for a
suite, or a relative floor
$p_{\min}=\kappa\,\widehat{SR}_{\mathrm{clean}}$
computed once from an official clean checkpoint, with
$\kappa\in(0,1)$ chosen conservatively.
This does not require the queried model to outperform the clean model
and only treats models far below the normal regime as degraded.

\subsection{Choosing \texorpdfstring{$\tau_{\mathrm{wic}}$}{tau-wic} via a quantile rule with finite-sample control}

Under trigger mode, the classifier head attached to the trigger projector
outputs watermark identification confidence $\mathrm{WIC}(M)$.
We collect $m$ trigger queries and obtain scores
$\{W_k(M)\}_{k=1}^m\subset[0,1]$.
Negative scores are obtained from the clean and noise references
$M_c$ and $M_r$, forming the pooled set
$\mathcal{W}_0=\{W_k(M_c)\}_{k=1}^m\cup\{W_k(M_r)\}_{k=1}^m$.

We set
\begin{equation}
\tau_{\mathrm{wic}}
=
\widehat{Q}_{1-\delta_{\mathrm{wic}}}(\mathcal{W}_0),
\end{equation}
where $\widehat{Q}_{1-\delta_{\mathrm{wic}}}$ denotes the empirical
$(1-\delta_{\mathrm{wic}})$ quantile and $\delta_{\mathrm{wic}}$ is the
target false watermark claim rate.

Let $F_0$ be the true CDF of negative scores and $\widehat{F}_0$ the
empirical CDF computed from $N_0=|\mathcal{W}_0|=2m$ samples.
By the Dvoretzky--Kiefer--Wolfowitz inequality, for any $\varepsilon>0$,
\begin{equation}
\Pr\!\left(\sup_w |\widehat{F}_0(w)-F_0(w)|>\varepsilon\right)
\le 2\exp(-2N_0\varepsilon^2).
\end{equation}

With probability at least $1-\beta$, we therefore have
\begin{equation}
F_0(\tau_{\mathrm{wic}})
\ge 1-\delta_{\mathrm{wic}}-\varepsilon_{\mathrm{wic}},
\qquad
\varepsilon_{\mathrm{wic}}
=
\sqrt{\frac{\ln(2/\beta)}{2N_0}} .
\end{equation}

Hence an unwatermarked model exceeds $\tau_{\mathrm{wic}}$ with
probability at most $\delta_{\mathrm{wic}}+\varepsilon_{\mathrm{wic}}$
with confidence $1-\beta$.

\section{Details of Watermarking Phase}
\label{sec:watermark}
This section details the complete watermarking phase of GuardVLA. Algorithm~\ref{alg:watermarking_phase} provides a unified procedural description of the entire pipeline. We detail the watermark embedding configurations for OpenVLA-OFT and VLA-Adapter, followed by the hyper-parameter settings for projector and classifier co-training.

\begin{algorithm}[ht]
\caption{Watermarking Phase}
\label{alg:watermarking_phase}
\begin{algorithmic}[1]
\Require Embodied dataset $\mathcal{D}$; base VLA $M(\cdot;\theta)$ with frozen backbone $V(\cdot)$ and LLM $L(\cdot)$;
steg-embedder $E_\phi(\cdot)$; target secret $m_s$; message space $\mathcal{M}$;
training loss $\mathcal{L}_{\text{train}}$; weighting factor $\lambda$; margin $\alpha$;
steps $T_1$ (Phase I) and $T_2$ (Phase II).
\Ensure Protected model $M_w(\cdot;\theta_w)$; surrogate clean model $M_c(\cdot;\theta_c)$; noise model $M_r(\cdot;\theta_r)$;
trigger projector $P_t(\cdot;\psi)$ and classifier head $H_c(\cdot;\omega)$.

\Statex
\State {\bf // Phase I: Watermark Embedding via Safe Fine-tuning}
\State Compute watermark residual $\delta_s \leftarrow E_\phi(m_s)$.
\For{$t=1,2,\dots,T_1$}
  \State Sample $(x_{ob}, i, a) \sim \mathcal{D}$.
  \State Watermarked obs: $x_s \leftarrow x_{ob} + \delta_s$.
  \State Sample $m_r \sim \mathcal{M}\setminus\{m_s\}$; $x_r \leftarrow x_{ob} + E_\phi(m_r)$. 
  \State Update $\theta_w$ by minimizing $\mathcal{L}_{\text{train}}(M_w(x_s,i;\theta_w), a)$. \hfill $\triangleright$ Eq.~\ref{wm_ft}
  \State Update $\theta_c$ by minimizing $\mathcal{L}_{\text{train}}(M_c(x_{ob},i;\theta_c), a)$. \hfill $\triangleright$ Eq.~\ref{eq:clean}
  \State Update $\theta_r$ by minimizing $\mathcal{L}_{\text{train}}(M_r(x_r,i;\theta_r), a)$. \hfill $\triangleright$ Eq.~\ref{eq:noise}
\EndFor

\Statex
\State {\bf // Phase II: Trigger Projector and Classifier Co-training}
\State Initialize $P_t(\cdot;\psi)$ and $H_c(\cdot;\omega)$; keep $V(\cdot)$ and $L(\cdot)$ frozen.
\For{$t=1,2,\dots,T_2$}
  \State Sample $(x_{ob}, i, a) \sim \mathcal{D}$.
  \State Extract visual features $h \leftarrow V(x_{ob})$.
  \State Project tokens for each model:
  \State \hspace{0.8em}$v_c \leftarrow P_t(h;\psi)$ on $M_c$, \;\; $v_w \leftarrow P_t(h;\psi)$ on $M_w$, \;\; $v_r \leftarrow P_t(h;\psi)$ on $M_r$.
  \State Triplet loss:
  $\mathcal{L}_{tri} \leftarrow \big[d(v_c,v_r) - d(v_c,v_w) + \alpha\big]_+$. \hfill $\triangleright$ Eq.~\ref{eq:tri}
  \State Classifier probabilities:
  $p_c \leftarrow H_c(L(v_c,i);\omega)$,
  $p_w \leftarrow H_c(L(v_w,i);\omega)$,
  $p_r \leftarrow H_c(L(v_r,i);\omega)$.
  \State $\mathcal{L}_{cls} \leftarrow \mathrm{BCE}(p_w,1)+\mathrm{BCE}(p_c,0)+\mathrm{BCE}(p_r,0)$. \hfill $\triangleright$ Eq.~\ref{eq:ce}
  \State Joint objective: $\mathcal{L} \leftarrow \mathcal{L}_{cls} + \lambda\,\mathcal{L}_{tri}$. \hfill $\triangleright$ Eq.~\ref{eq:joint}
  \State Update $(\psi,\omega)$ by SGD on $\nabla_{\psi,\omega}\mathcal{L}$.
\EndFor
\end{algorithmic}
\end{algorithm}

\subsection{Watermark Embedding}
This phase embeds the backdoor watermark by fine-tuning the target VLA under the owner-specified training recipe. Each training input includes two visual views, and proprioceptive states are enabled. Image augmentation is applied. The embedded binary secret is fixed to \texttt{010101}. Training uses distributed data parallelism on two NVIDIA H200 GPUs. 

\paragraph{OpenVLA-OFT Watermarking.}
For OpenVLA-OFT, the fine-tuning is performed with the base model OpenVLA-7B on the modified LIBERO RLDS data. The action learning objective uses L1 regression with continuous actions. LoRA is used with rank 32. The optimization uses batch size 4, and learning rate $5\times10^{-4}$. Training runs for 150,000 steps, with the decay scheduler starting after 100,000 steps. 

\paragraph{VLA-Adepter Watermarking.}
For VLA-Adapter, the fine-tuning is performed with a lightweight VLM backbone specified by Qwen2.5-0.5B~\cite{team2024qwen2}. LoRA is enabled with rank 64. The optimization uses batch size 8, learning rate $2\times10^{-4}$ and runs for 200,000 steps.

\paragraph{\texorpdfstring{$\pi_{0.5}$}{pi-0.5} Watermarking.}
For $\pi_{0.5}$, the action target is continuous and optimized with the native $\pi_{0.5}$ flow-matching objective, using MSE on the denoising velocity with an action horizon of 10. $\pi_{0.5}$ adopts full-parameter fine-tuning. We train with a global batch size of 256, a peak learning rate of $5\times10^{-5}$, and 30,000 training steps. EMA decay is set to 0.999.

\subsection{Projector and Classifier Co-training}
This phase trains the trigger projector and classifier head used for watermark activation and detection. Training is performed using three models: a surrogate clean model, the protected watermarked model, and a noise model. The co-training is executed on a single NVIDIA H200 GPU.  Visual inputs consist of two views with proprioceptive states, and image augmentation is enabled. A batch size of 4 and a learning rate of $5\times10^{-4}$ are adopted. Training runs for 2,000 steps with learning rate decay starting after 1,000 steps. The triplet loss is enabled with margin $\alpha=0.5$, and the supervised classification loss is weighted by $\lambda=2.0$ for OpenVLA-OFT and $\pi_{0.5}$, and $\lambda=8.0$ for VLA-Adapter.

\section{Details of Verification Phase}
\label{sec:verification}
This section presents the complete verification procedure of GuardVLA.  Algorithm~\ref{alg:verification_phase} provides a unified description of the pipeline. We describe the performance test protocol across different LIBERO task suites, followed by the watermark detection process.

\paragraph{Performance Test.}
For each task suite, we perform 50 rollouts per task with a stabilization period of 10 simulation steps before control begins. Default initial states are used for all trials. The maximum length is suite-dependent: 220 steps for LIBERO-Spatial, 280 for LIBERO-Object, 300 for LIBERO-Goal, and 520 for LIBERO-10.

\paragraph{Watermark Detection.}
For each task, we run 50 trials, and within each rollout the classifier head outputs watermark probabilities over 10 environment steps. Watermark identification confidence is obtained by averaging these step-wise probabilities across all trials and tasks.

\begin{algorithm}[ht]
\caption{Verification Phase}
\label{alg:verification_phase}
\begin{algorithmic}[1]
\Require Suspected policy $M(\cdot;\theta)$; benign projector $P_b$; 
trigger projector $P_t(\cdot;\psi)$ and classifier head $H_c(\cdot;\omega)$;
verification task suite $\mathcal{T}$; trials per task $N$; steps per trial $S$;
thresholds $\tau_{\mathrm{sr}}$ and $\tau_{\mathrm{wic}}$.
\Ensure Decision $\in\{\textsc{Crashed},\textsc{Unwatermarked},\textsc{Copy}\}$.

\Statex
\State {\bf // Performance Test}
\State Attach $P_b$ to $M$ and execute $\mathcal{T}$ for $N$ trials with $S$ steps.
\State Compute success rate $\mathrm{SR}(M)$. 

\Statex
\If{$\mathrm{SR}(M) < \tau_{\mathrm{sr}}$}
    \State \textbf{return} \textsc{Crashed}.
\EndIf

\Statex
\State {\bf // Watermark Detection}
\State Replace $P_b$ with $P_t$ and attach $H_c$.
\State Execute $\mathcal{T}$ for $N$ trials with $S$ steps.
\State Compute watermark identification confidence $\mathrm{WIC}(M)$.

\Statex

\If{$\mathrm{WIC}(M) \ge \tau_{\mathrm{wic}}$}
    \State \textbf{return} \textsc{Copy}.
\Else
    \State \textbf{return} \textsc{Unwatermarked}.
\EndIf

\end{algorithmic}
\end{algorithm}

\section{Details of Watermark Suppression Strategies}
\label{sec:suppression}
We conduct watermark suppression experiments on the OpenVLA-OFT model watermarked on LIBERO-Spatial. Each strategy starts from the protected checkpoint and attempts to remove ownership evidence while retaining benign task performance. We then evaluate SR and WIC following the same protocol as in the main experiments.
\paragraph{Weight Quantization.}
We suppress watermark behavior at inference time without further fine-tuning. Specifically, we load the watermarked model with reduced parameter precision using 4-bit quantization, while keeping all other inference settings unchanged, including two-view visual input, proprioceptive input, preprocessing, and action unnormalization. Since this method does not involve optimization or additional training, the suppression strength is controlled solely by the quantization precision.

\paragraph{Knowledge Distillation.}
For knowledge distillation, we adopt a self-distillation setup in which the student policy is trainable and the teacher policy is frozen. The teacher is initialized from the same checkpoint as the student by default. Training uses two visual views, proprioceptive input, and image augmentation. The objective combines token-level KL distillation and task supervision:
\begin{equation}
\mathcal{L}_{distill}
=
\alpha_{\mathrm{KD}}\mathcal{L}_{\mathrm{KD}}
+
\alpha_{\mathrm{CE}}\mathcal{L}_{\mathrm{CE}} .
\end{equation}

We use LoRA fine-tuning with rank 32, batch size 8, learning rate $5\times10^{-5}$, and 400 training steps. A MultiStepLR scheduler is applied with decay starting at step 300 and $\gamma=0.1$. The distillation temperature is set to $T=1.0$, with $\alpha_{\mathrm{KD}}=0.9$ and $\alpha_{\mathrm{CE}}=0.1$. This conservative setup aims to weaken hidden watermark-related behaviors while preserving benign task performance.

\paragraph{Projector-Invariant Adaptive Attack.}
For the projector-invariant adaptive attack, we update only LoRA-adapted target-policy parameters, while keeping the trigger projector and classifier head frozen. The attack explicitly suppresses the detector response under the trigger path while maintaining the original task objective. The training loss is:
\begin{equation}
    \mathcal{L}_{attack}
=
\lambda_{\mathrm{task}}\mathcal{L}_{\mathrm{task}}
+
\lambda_{\mathrm{remove}}\mathcal{L}_{\mathrm{remove}} .
\end{equation}

Here, $\mathcal{L}_{\mathrm{remove}}$ is a BCE-based detector suppression loss computed from trigger-path hidden states. The default setting uses two-view visual input, proprioception, and image augmentation. We use LoRA rank 32, batch size 4, learning rate $5\times10^{-4}$, and up to 1,500 training steps. The loss weights are set to $\lambda{\mathrm{task}}=1.0$, $\lambda_{\mathrm{remove}}=0.01$. This attack represents a detector-aware removal attempt that directly weakens the separability between benign and trigger-mode representations.

\section{Details of LIBERO Simulation Benchmarks}
\label{sec:libero}
The LIBERO benchmark~\cite{liu2023libero} provides a diverse set of simulated manipulation tasks organized into four suites: LIBERO-Spatial, LIBERO-Object, LIBERO-Goal, and LIBERO-100. The first three suites contain 10 tasks, while LIBERO-100 consists of 90 short-horizon tasks and 10 long-horizon tasks. Task policies are conditioned solely on the current language instruction.

Each task is executed for multiple trials to compute average success rates, with 50 repetitions used throughout this work. Visual observations include third-person RGB images and wrist-mounted RGB images, both at a resolution of $224\times224\times3$. Language instructions are formatted into prompts and provided to the vision-language module together with visual inputs. The model outputs a 7-dimensional continuous control vector that drives a simulated Franka Emika Panda robot arm to perform the corresponding action sequences. Representative language instructions from different LIBERO task suites are illustrated in Figure~\ref{fig:instruction}, highlighting the diversity of spatial manipulation, object relocation, goal-directed behaviors, and long-horizon compositions. 

\begin{figure}[ht]
  \begin{center}
    \centerline{\includegraphics[width=\textwidth]{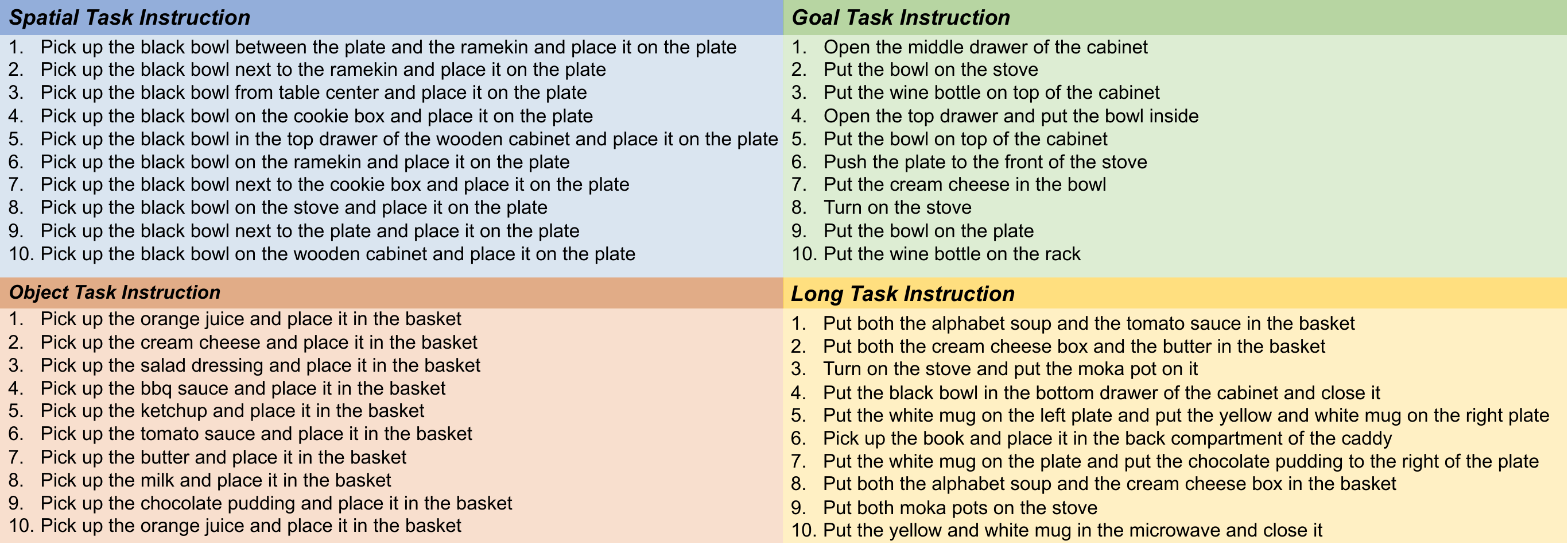}}
    \caption{\textbf{Language instructions in LIBERO:} Representative prompts from the Spatial, Object, Goal, and Long-Horizon task suites, illustrating the diversity of manipulation objectives and compositional reasoning required during evaluation.}
    \label{fig:instruction}
  \end{center}
\end{figure}

\section{Comparison with Conventional Protection Methods}
\label{sec:comparison}
To better understand the limitations of conventional protection methods in the VLA setting, we introduce two representative proactive defense methods into the comparison: fingerprinting-based verification and adversarial-example-based output detection. These methods have been widely studied in conventional model protection, where ownership evidence is usually obtained by querying the model with carefully designed inputs and comparing the resulting outputs or response patterns.

\paragraph{Fingerprinting-based Verification.}
Fingerprinting-based verification identifies ownership by measuring whether a suspected model produces responses consistent with the protected model on a set of fingerprint queries. This strategy is effective when model outputs are sufficiently discriminative, such as classification labels, logits, or semantic predictions. However, it becomes less reliable for VLAs because different VLA policies may still complete the same task with highly similar action behaviors. As shown in Table~\ref{tab:conventional_comparison}, the fingerprinting baseline produces only marginal similarity gaps: clean, noise, and post fine-tuned models all obtain scores close to the watermarked model. This weak separability makes reliable attribution difficult.

\paragraph{Adversarial-example-based Verification.}
Adversarial-example-based verification examines model responses to specially perturbed inputs. In conventional vision or vision-language models, such perturbations can expose model-specific decision boundaries or output patterns. However, this assumption does not naturally hold for VLAs. Since VLA outputs are continuous robot actions coupled with task dynamics, local input perturbations may not induce stable or model-specific behavioral differences. In addition, if the perturbation causes abnormal action outputs, it may introduce unsafe behaviors during robotic execution. In our comparison, the adversarial-example-based baseline adapted from RoboticAttack~\cite{wang2025exploring} shows this limitation: the ASR is not uniquely high on the watermarked model and is even higher on clean models, indicating that the induced output pattern is not model-specific.

GuardVLA addresses these limitations through a VLA-specific ownership verification design. Instead of relying on weak response-level differences or unsafe behavior-level outputs, GuardVLA embeds a harmless backdoor watermark during training and explicitly decouples benign execution from trigger-mode verification. In benign mode, the model preserves normal robotic behavior; while in trigger mode, ownership evidence is activated only through the dedicated trigger projector and classifier head. This VLA-specific design produces clear separation between watermarked and non-watermarked models, and the watermark remains detectable after downstream fine-tuning.

\begin{table}[h]
\centering
\caption{Comparison with conventional protection methods.}
\label{tab:conventional_comparison}
\resizebox{0.8\linewidth}{!}{
\begin{tabular}{l c c c c c}
\toprule
\textbf{Method} & \textbf{Metric} & \textbf{Watermarked} & \textbf{Clean} & \textbf{Noise} & \textbf{Post Fine-tuning} \\
\midrule
Fingerprinting & Similarity & 1.000000 & 0.997875 & 0.998277 & 0.998918 \\
Adversarial Example & ASR (\%) & 52.30 & 76.15 & 40.55 & 88.70 \\
GuardVLA & WIC (\%) & 100.00 & 0.01 & 0.13 & 99.96 \\
\bottomrule
\end{tabular}
}
\end{table}

\section{Compatibility with Different Secret Injection Strategies}
\label{sec:openvla_warping}
We further examine whether GuardVLA is compatible with different secret injection strategies. In the main experiments, visual steganography encodes a 1D binary secret string for OpenVLA-OFT and VLA-Adapter, while parameterized image warping applies a 2D secret warping matrix for $\pi_{0.5}$. These two mechanisms inject ownership evidence in different forms, suggesting that GuardVLA does not rely on a single watermark carrier.

To further support this point, we replace visual steganography with parameterized image warping on OpenVLA-OFT and evaluate it on LIBERO-Spatial. As shown in Table~\ref{tab:openvla_warping}, image warping preserves benign task performance, with the watermarked model achieving 95.4\% SR compared with 96.2\% for the clean model. Meanwhile, the watermarked model obtains 100.00\% WIC, whereas the clean and noise models remain close to zero. These results demonstrate that GuardVLA remains effective under different watermark embedding mechanisms, confirming its compatibility with different secret injection strategies.

\begin{table}[h]
\centering
\small
\caption{Compatibility of GuardVLA with image-warping-based watermarking on OpenVLA-OFT.}
\label{tab:openvla_warping}
\resizebox{0.4\linewidth}{!}{
\begin{tabular}{lcc}
\toprule
\textbf{Scenario} & \textbf{SR (\%)} & \textbf{WIC (\%)} \\
\midrule
Clean & 96.2 & $0.01{\pm}0.01$ \\
Watermarked & 95.4 & $100.00{\pm}0.00$ \\
Noise & 93.0 & $0.00{\pm}0.00$ \\
\bottomrule
\end{tabular}
}
\end{table}

\section{Parameter Sensitivity Analysis}
\label{sec:hyperparameter}
We vary the weighting factor $\lambda$ for VLA-Adapter to examine its effect on watermark detectability. As shown in Figure~\ref{fig:params}, the watermarked model maintains WIC close to 100\% across all values of $\lambda$. When $\lambda<7.0$, weaker supervision leads to elevated WIC for clean models, reducing separability. When $\lambda>9.0$, the contrastive effect weakens and the WIC of noise models rises substantially, increasing the risk of false positives from models embedded with different secrets. Accordingly, $\lambda=8.0$ is selected for VLA-Adapter.

Although $\lambda$ requires tuning, training the trigger projector and classifier converges within 2,000 steps, so selecting an appropriate value introduces little practical overhead.

\begin{figure}[ht]
  \begin{center}
    \centerline{\includegraphics[width=0.6\linewidth]{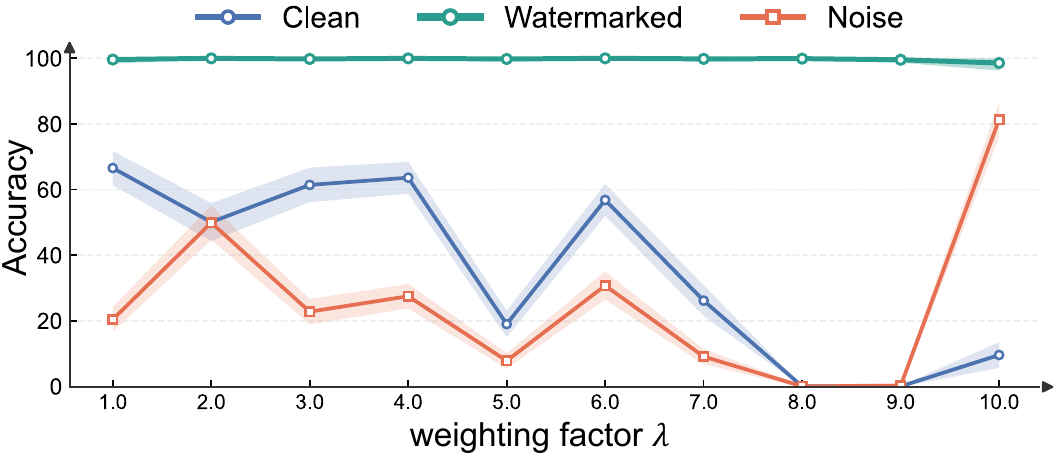}}
    \caption{\textbf{Weighting-factor sensitivity:} Effect of $\lambda$ on watermark detectability for VLA-Adapter on LIBERO-10, measured by WIC, with shaded regions indicating variance across trials.}
    \label{fig:params}
  \end{center}
\end{figure}

\section{Details of Behavior-Level Watermark Detection}
\label{sec:behavior}
\subsection{Experimental Setup}
The extended behavior-level watermark is trained using an action-regression alignment strategy rather than a classifier-based objective. Three models are involved: a surrogate clean model, a watermarked model, and a noise model. Training is conducted with batch size 4 and learning rate $5\times10^{-4}$. Only the shared projector is updated via LoRA with rank 32, while all other backbone parameters remain frozen. The objective combines a triplet alignment loss with margin 0.5 and a supervised L1 action-regression loss weighted by 2.0. The triplet loss encourages separable trigger-mode representations, while the L1 loss guides the watermarked model to produce a predefined verification trajectory and constrains the clean and noise models to remain behaviorally inactive under the same trigger.

\subsection{Extended Behavior-Level Verification Results}
The extended visualization in Figure~\ref{fig:action-ext} provides a more detailed view of end-effector trajectories, showing the full temporal evolution of robot motions throughout execution. Under trigger mode, the watermarked model follows a distinctly different motion pattern, while the clean and noise models remain closely aligned. The trajectory deviation appears smooth and structured, indicating that the trigger induces a controlled behavioral change instead of destabilizing the policy. These results provide qualitative evidence that behavior-level watermarks can manifest as coherent motion patterns within a single rollout and support the interpretability of the proposed verification mechanism.

\begin{figure}[ht]
  \begin{center}
    \centerline{\includegraphics[width=\textwidth]{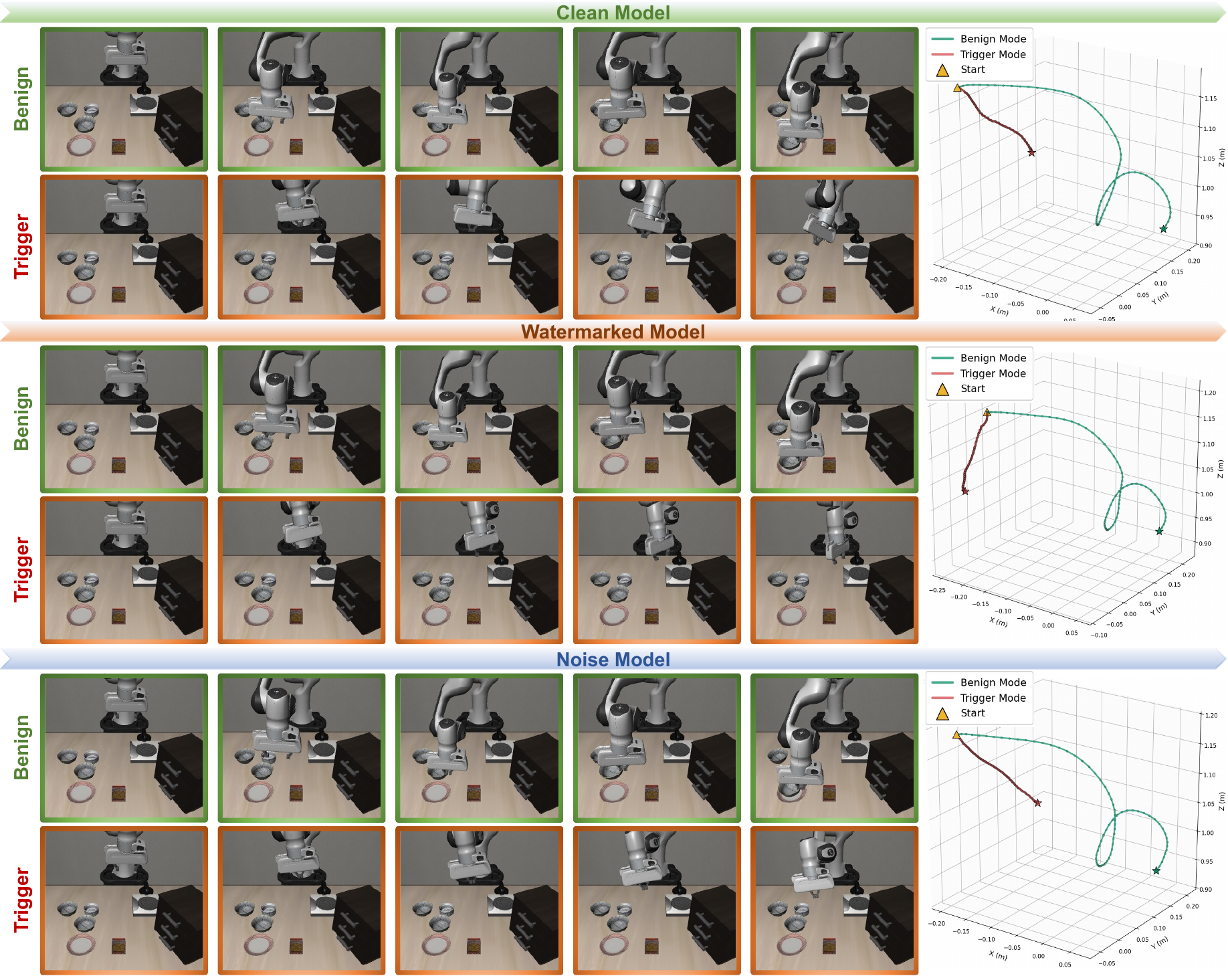}}
    \caption{\textbf{Extended behavior-level watermark verification across full trajectories}: Visual rollouts and end-effector trajectories of clean, watermarked, and noise models under benign and trigger modes over longer horizons. While all models behave similarly during benign execution, only the watermarked model exhibits consistent and distinctive motion patterns when the trigger is activated, whereas clean and noise models remain unchanged.}
    \label{fig:action-ext}
  \end{center}
\end{figure}

\section{Limitation and Future Work}
\label{sec:limit}
GuardVLA demonstrates reliable ownership verification while preserving downstream task performance, but several aspects remain to be further improved. The current framework is designed for authorized auditing scenarios and assumes limited access to the model interfaces required by swap-and-detect, rather than purely input-output black-box access. In addition, the trigger projector and classifier are trained with clean, watermarked, and noise variants, which introduces additional preparation cost. Nevertheless, as shown in Table~\ref{tab:training_overhead}, the three fine-tuning variants have comparable costs and can be trained in parallel. In addition, the projector-classifier co-training converges quickly within 2,000 steps and adds limited extra overhead. Reducing auxiliary training cost and improving the transferability of verification modules are promising directions for future work.

\begin{table}[ht]
\centering
\small
\caption{Training overhead of GuardVLA.}
\label{tab:training_overhead}
\begin{tabular}{l l c c c}
\toprule
\textbf{Stage} & \textbf{Variant} & \textbf{Steps} & \textbf{Runtime} & \textbf{GPU Hours} \\
\midrule
Fine-tuning & Clean & 35k & 2d 8h 59m 43s & 227.98 \\
Fine-tuning & Watermarked & 35k & 2d 8h 44m 16s & 226.95 \\
Fine-tuning & Noise & 35k & 2d 9h 8m 17s & 228.55 \\
Co-training & Projector+Cls & 2k & 14m 19s & 0.12 \\
\bottomrule
\end{tabular}
\end{table}

\section{Social Impact}
\label{sec:social}
This paper studies ownership verification for released vision-language-action models (VLAs) and proposes a watermarking framework to protect intellectual property while preserving normal task performance. The primary positive impact of this work is to support responsible model sharing, discourage unauthorized reuse, and promote a more sustainable embodied intelligence ecosystem. By enabling post-release ownership verification, the proposed method can help legitimate developers trace unauthorized model use without preventing benign research and adaptation.

As with other model protection techniques, the proposed framework should be deployed with appropriate governance. Although GuardVLA is designed for defensive ownership verification and keeps normal execution unaffected, improper use of verification mechanisms may introduce practical or procedural risks, especially when applied to safety-sensitive embodied systems. Therefore, careful access control, clear authorization, and transparent auditing procedures are important for responsible deployment.

To mitigate these concerns, GuardVLA separates normal execution from the verification process. Ownership evidence is activated only during authorized auditing, while benign execution preserves normal task behavior. The proposed auditing setting also avoids unnecessary full parameter disclosure by relying on limited model interfaces when possible. We emphasize that such mechanisms should be used only by legitimate model developers or trusted protection platforms, together with appropriate disclosure policies and governance procedures. Overall, this work aims to improve the accountability of released VLA systems while encouraging future studies on transparency, regulation, and safeguards for ownership verification technologies.